\documentclass[10pt,twocolumn,letterpaper]{article}

\usepackage{wacv}
\usepackage{times}
\usepackage{epsfig}
\usepackage{graphicx}
\usepackage{amsmath}
\usepackage{amssymb}
\usepackage{amsthm}
\newtheorem{theorem}{Theorem}
\graphicspath{{images/}} 
\usepackage{algpseudocode}
\usepackage{algorithm}
\usepackage{color}
\usepackage{xcolor}
\usepackage{subfigure}
\usepackage{wrapfig}
\usepackage{bm}
\usepackage{caption}
\usepackage{mathtools, cuted}

\definecolor{bcolor}{RGB}{50, 125, 50}
\usepackage{adjustbox}
\usepackage{lipsum}
\usepackage{caption}
\definecolor{bcolor}{RGB}{50, 125, 50}
\usepackage{adjustbox}

\usepackage{url}

\wacvfinalcopy 

\def\wacvPaperID{317} 

\ifwacvfinal\fi
\begin{document}

\title{Generating Positive Bounding Boxes for Balanced Training of Object Detectors}

\author{Kemal Oksuz, Baris Can Cam,  Emre Akbas$^*$, Sinan Kalkan\thanks{Equal contribution for senior authorship.}\\
Department of Computer Engineering \\
Middle East Technical University, Ankara, Turkey\\
{\tt\small \{kemal.oksuz, can.cam, eakbas, skalkan\}@metu.edu.tr}
}

\maketitle
\ifwacvfinal\thispagestyle{empty}\fi


\begin{abstract}
Two-stage deep object detectors generate a set of regions-of-interest (RoI) in the first stage, then, in the second stage, identify objects among the proposed RoIs that sufficiently overlap with a ground truth (GT) box. The second stage is known to suffer from a bias towards RoIs that have low intersection-over-union (IoU) with the associated GT boxes. To address this issue, we first propose a sampling method to generate bounding boxes (BB) that overlap with a given reference box more than a given IoU threshold. Then, we use this BB generation method to develop a positive RoI (pRoI) generator that produces RoIs following any desired spatial or IoU distribution, for the second-stage. We show that our pRoI generator is able to simulate other sampling methods for positive examples such as hard example mining and prime sampling. Using our generator as an analysis tool, we show that (i) IoU imbalance has an adverse effect on performance, (ii) hard positive example mining improves the performance only for certain input IoU distributions, and (iii) the imbalance among the foreground classes has an adverse effect on performance and that it can be alleviated at the batch level. Finally, we train Faster R-CNN using our pRoI generator and, compared to conventional training, obtain better or on-par performance for low IoUs and significant improvements when trained for higher IoUs for Pascal VOC and MS COCO datasets. The code is available at: \url{https://github.com/kemaloksuz/BoundingBoxGenerator}.
\end{abstract}

\section{Introduction}
\label{sec:intro}
\begin{figure}
\subfigure[Bounding Box Generation]{
    \centerline{
    \includegraphics[width=0.4\textwidth]{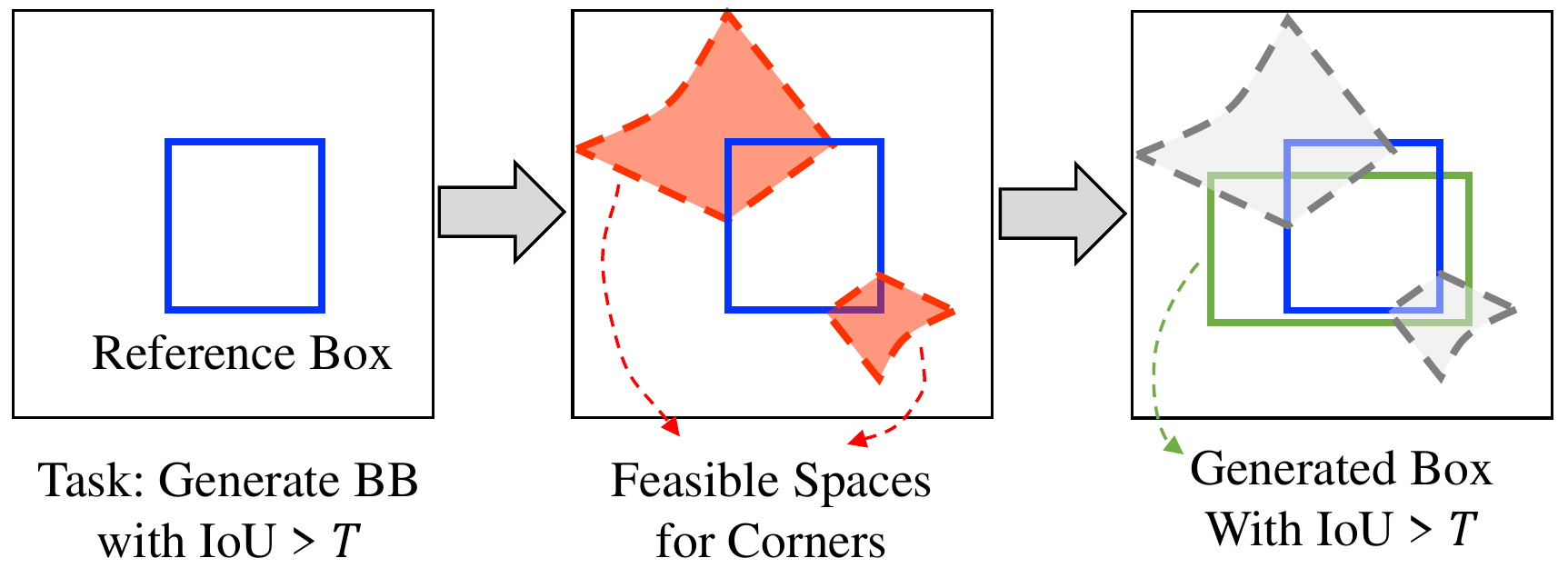}
    }
    \label{fig:BBteaser}
    }
\subfigure[Generating Bounding Boxes for Training an Object Detector]{
    \centerline{\includegraphics[width=0.46\textwidth]{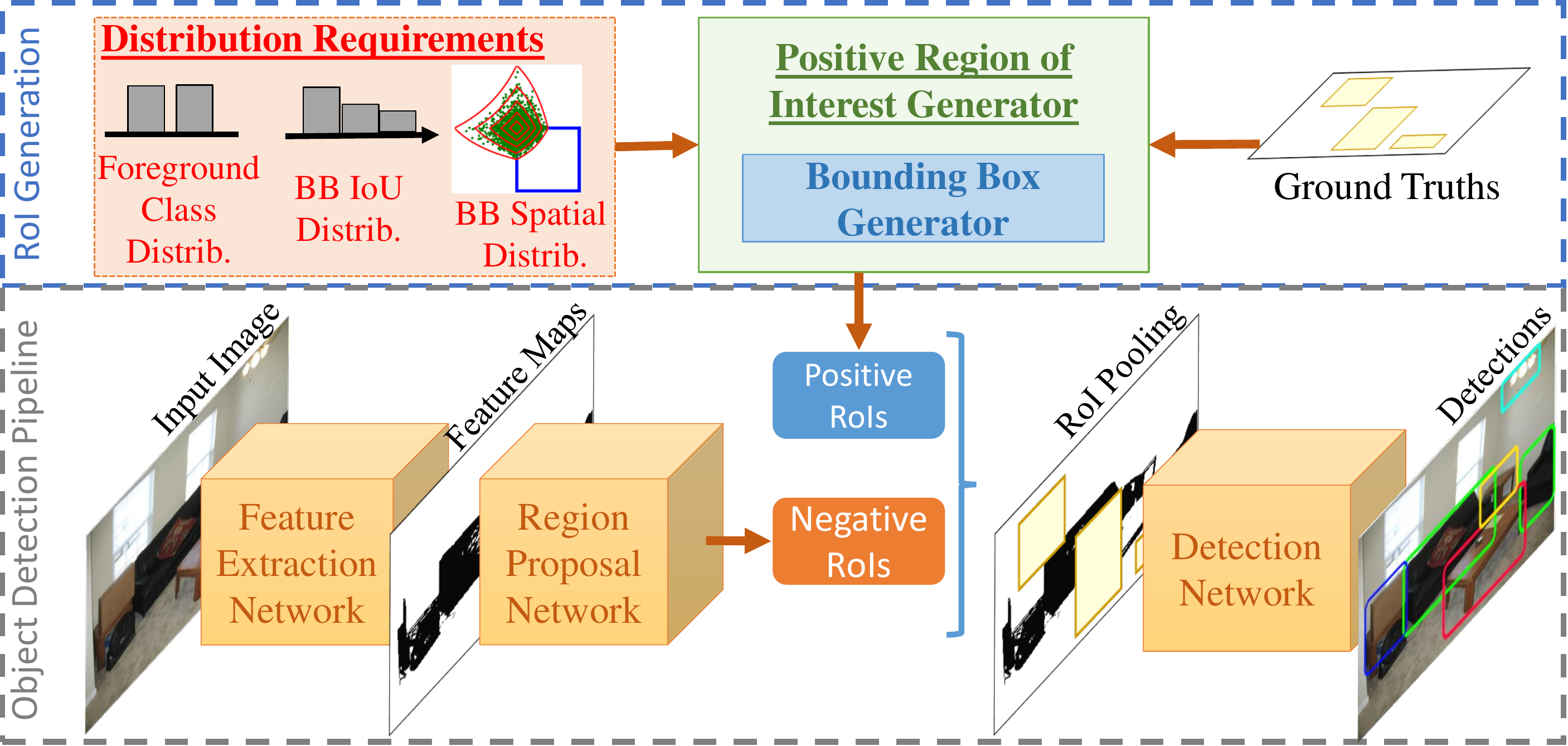}} \label{fig:teaser}
    }
\caption{\textbf{(a)} \textit{An illustration of Bounding Box (BB) Generation.} Given a reference box (in blue) and an IoU threshold $T$, a BB having at least $T$ IoU is generated (drawn in green). \textbf{(b)} \textit{An illustration of training an object detector with positive region-of-interests.} Given distribution requirements on foreground classes and BBs, we generate positive RoIs using the BB generator (Fig \ref{fig:BBteaser}). Negative RoIs are still generated by the region proposal network.}
\end{figure}



An important challenge in object detection is class imbalance \cite{PrimeSample,gradientharmonizing,FocalLoss,LibraRCNN,OHEM}: even from a single image, an infinite number of negative examples can be sampled, in contrast to only a limited set of positive RoIs. Naturally, this leads to significant imbalance between negative and positive RoIs. Class imbalance also exists within foreground classes.

A prominent solution to the foreground-background class imbalance is to have two stages \cite{RFCN,FastRCNN,FasterRCNN}: The first stage estimates regions (i.e., region-of-interests -- RoIs) that are likely to contain objects, significantly discarding background samples, and the second-stage classifies these regions into objects, and also fine-tunes the coordinates of the bounding boxes.  Other solutions generally employ sampling with hard constraints (e.g., online hard example mining \cite{OHEM}, Libra RCNN \cite{LibraRCNN}) or soft constraints (e.g., focal loss \cite{FocalLoss}, harmonizing gradients \cite{gradientharmonizing}). 

The foreground-foreground class imbalance problem, i.e., the imbalance in the number of examples pertaining to different positive classes at the image, dataset or mini-batch levels, has not attracted as much attention. In addition, the intersection-over-union (IoU) distribution of the RoIs generated by the region proposal network (RPN) \cite{FasterRCNN} is imbalanced \cite{CascadeRCNN}, which biases the BB regressor in favor of the IoU that the distribution is skewed towards. We call this imbalance problem as \textit{BB IoU imbalance}. Addressing these problems requires a careful analysis of the positive RoIs. 

In this paper, we analyze and address foreground-foreground class imbalance and BB IoU imbalance by actively generating BBs. We first propose the ``BB generator", a method that can generate an arbitrary BB overlapping with a reference box with an IoU larger than a given threshold (Figure \ref{fig:BBteaser}). Using the BB generator, we develop a positive RoI (pRoI) generator that can produce RoIs conforming to desired BB IoU and spatial distributions (Figure \ref{fig:teaser}). Considering that there is a correlation between the hardness of an example and its IoU \cite{LibraRCNN}, the pRoI generator is able to \textit{generate} (rather than sample) not only positive or negative samples, but also samples with any desired property such as hard examples \cite{OHEM} or prime samples \cite{PrimeSample}. 

We use our pRoI generator to perform several analyses and improvements. Specifically, we (i) show that BB IoU and foreground class distributions affect performance, (ii)  make a comparative analysis for RPN RoIs and (iii) improve the performance of Faster RCNN for IoU intervals where RPN is not able to generate enough samples. 

Finally, we devise an online, foreground-balanced (OFB) sampling method which considers the imbalance among the foreground classes dynamically within a training batch based on multinomial sampling. 

Overall, our contributions can be summarized as follows:

\noindent\textit{1. Generators:} (i) A BB generator to generate BBs for a given IoU threshold and (ii) a positive RoI generator to generate RoIs with desired foreground class, IoU and spatial distributions. 

\noindent\textit{2. Imbalance Problems and Analysis:} We introduce the BB imbalance problem. Using our pRoI generator, we show that these imbalance problems and the foreground-foreground class imbalance within a training batch affect the performance of the object detectors. We also provide an analysis of RPN RoIs and show that the effect of the hard examples depends on the IoU distribution of the BBs. 

\noindent\textit{3. Practical Improvements:} We train a detection network using our pRoI generator, which increases the amount and the diversity of the positive examples especially for the larger IoUs, and show that the performance improves compared to the standard training (e.g. for $IoU=0.8$, $\mathrm{mAP@0.8}$ improves by $10.9\%$ for Pascal VOC). We also train the conventional detection pipeline by using the proposed OFB sampling, and improve the performance. 

\section{Related Work}
\label{sec:RelatedWork}

{\noindent}\textbf{Deep Object Detectors:} We can group deep object detectors into two: One-stage methods and two-stage methods. While one-stage methods \cite{DSSD, FocalLoss, SSD, YOLO, YOLObetter} predict the object categories and their bounding boxes directly from anchors, two-stage methods \cite{RFCN, FastRCNN, RCNN, FasterRCNN} first estimate a set of RoIs from anchors and then predict objects from these RoIs in the second stage. Both approaches use a deep feature extractor \cite{Resnet, Resnext}, optionally followed by steps like feature pyramid networks \cite{NASFPN, PyramidReconfiguration, FeaturePyramidNetwork, PANet}.

Our BB sampling approach is more suitable for the second stage of the two-stage methods since one-stage detectors have structural constraints owing to the fact that each output of a one-stage detector corresponds to a predefined anchor having fixed location, shape and scale. For this reason, an additional module is required to employ our generator. However, BB imbalance and OFB sampling are relevant for any object detection pipeline since any object detector needs to deal with bounding boxes even if they are estimated or fixed (in the case of anchors).  

{\noindent}\textbf{Class Imbalance in Object Detection:} Following the work by Oksuz et al. \cite{imbalance}, we categorize the class imbalance problem for the deep object detectors into two: foreground-background class imbalance and foreground-foreground class imbalance.

Foreground-background class imbalance has attracted more attention from the community with \textit{hard sampling}, \textit{soft sampling} and \textit{generative approaches}. In hard sampling, certain samples are shown more to the network to address imbalance. This can be performed e.g. via random sampling \cite{RFCN, FasterRCNN}, or by relying on ``sample usefulness'' heuristics as in hard-example mining \cite{SSD, LibraRCNN, OHEM} and prime sampling \cite{PrimeSample}. Hard-example mining methods usually assume that examples with higher loss are more difficult to learn, and therefore, they train a network more with such examples. This approach is adopted for negative samples in SSD \cite{SSD}, while a more systematic approach considering both the positive and negative samples is proposed in online hard example mining (OHEM -- \cite{OHEM}). An alternative hardness definition was proposed in Libra R-CNN \cite{LibraRCNN} based on a sample's IoU, and a solution was proposed using hard example mining using BB IoUs without computing the loss for the entire set. A recent interesting method, ``prime sampling" \cite{PrimeSample}, asserts that positive samples with higher IoUs are more representative and proposed ranking the positive samples based on its IoU with the ground truth, while still showing that hard example mining for the negative class works well. BB IoU imbalance is addressed by Cascade R-CNN \cite{CascadeRCNN}  by employing cascaded detectors in such a way that a later-stage detector is trained by a distribution skewed towards higher IoU. However, this, requiring multiple detectors being trained, is computationally prohibitive.

In \textit{soft sampling}, a weight is assigned to each sample rather than performing a discrete (hard) selection of samples. Prominent examples include focal loss \cite{FocalLoss}, which promotes hard examples; prime sampling \cite{PrimeSample}, which assigns more weight to examples with higher IoUs; and finally gradient harmonizing mechanism \cite{gradientharmonizing}, which assigns lower weights to easy negatives and suppresses the effect of the outliers. 

The \textit{generative methods} address imbalance with a different perspective by introducing generated samples. Example approaches include generating hard examples with various deformations and occlusion \cite{AFastRCNN} and generating synthetic examples \cite{GAN}.

Foreground-foreground class imbalance is critical as well. Kuznetsova et al. \cite{OpenImages} showed that object detection datasets are highly imbalanced also for foreground classes. The only method to consider the problem at the dataset level handcrafts a similarity measure, and based on the measure clusters the classes to have a more balanced training \cite{FgClassImbalance}. In the classification domain where there is no background class, this imbalance is studied more \cite{ImbalanceBook,ImbalanceSurvey1} by, e.g.,  performing class-aware sampling \cite{relay2016}. However, these methods are not directly applicable for two-stage object detectors because the second stage's input is very dynamic since it depends on RoIs estimated by the first stage. Despite this difference, class-aware sampling is said to be adopted by \cite{PANet}, however no comparison is presented for balanced and imbalanced training from the object detection perspective.


The ideas in this paper are relevant for both foreground-background and foreground-foreground class imbalance. One can generate any number of positive RoIs to address the foreground-background imbalance, and the generated set can also be chosen equally from each class to address the foreground-foreground imbalance. Among the three types of  methods mentioned above, we classify our approach as a generative method. Since the end-to-end training pipeline is not disrupted (see Figure \ref{fig:teaser}), any hard sampling method \cite{OHEM,LibraRCNN} can also be simulated. In addition, we directly address foreground-foreground class imbalance by online foreground balanced (OFB) sampling. Its main difference from the previously proposed class-aware sampling \cite{relay2016} is that while they use a static dataset, our OFB sampling is able to handle the dynamic nature of the RoIs (i.e. the batch depends on the sampled RoIs at each iteration) owing to the proposal network. 
\section{The Generators}
In this section, we describe the methods for generating bounding boxes and balanced positive RoIs.

\subsection{Definitions and Notation}
Let $B=[x_1,y_1,x_2,y_2]$ denote a ground-truth box with top-left corner $\mathrm{TL}(B)=(x_1,y_1)$ and bottom-right corner $\mathrm{BR}(B)=(x_2,y_2)$ satisfying $x_2>x_1$ and $y_2>y_1$. The area of $B$ is simply defined as:
\begin{equation}
\mathrm{A}(B)=(x_2-x_1)\times (y_2-y_1),    
\end{equation}
and the area of the intersection between $B$ and $\bar{B}$ is:
{
\begin{eqnarray}
 \mathrm{I}(B,\bar{B})& = & (\min{(\Bar{x_2},x_2)}-\max{(\Bar{x_1},x_1)}) \times  \\ \nonumber 
  & &  \quad (\min{(\Bar{y_2},y_2)}-\max{(\Bar{y_1},y_1)} .   
\end{eqnarray}
}
Based on this notation, $\mathrm{IoU}(B,\Bar{B})$ can be easily defined as:
\begin{align}
\label{eq:IoU}
\mathrm{IoU}(B,\Bar{B})=\frac{\mathrm{I}(B,\Bar{B})}{\mathrm{A}(B)+\mathrm{A}(\Bar{B})-\mathrm{I}(B,\Bar{B})}.
\end{align}
Finally, we note two useful properties of the IoU function: ({\bf Theorem 1}) $\mathrm{IoU}(B,\Bar{B})$ is scale-invariant, and ({\bf Theorem 2}) $\mathrm{IoU}(B,\Bar{B})$ is translation-invariant (see Appendix for proofs). These theorems allow us to shift and scale the input BBs to a reference box during BB generation and then shift and scale them back to their original aspect ratio and location..

\subsection{Bounding Box Generator}
\label{sec:BBSampler}

\begin{figure}
\centering
\begin{tabular}{cc}
\includegraphics[width=0.2\textwidth]{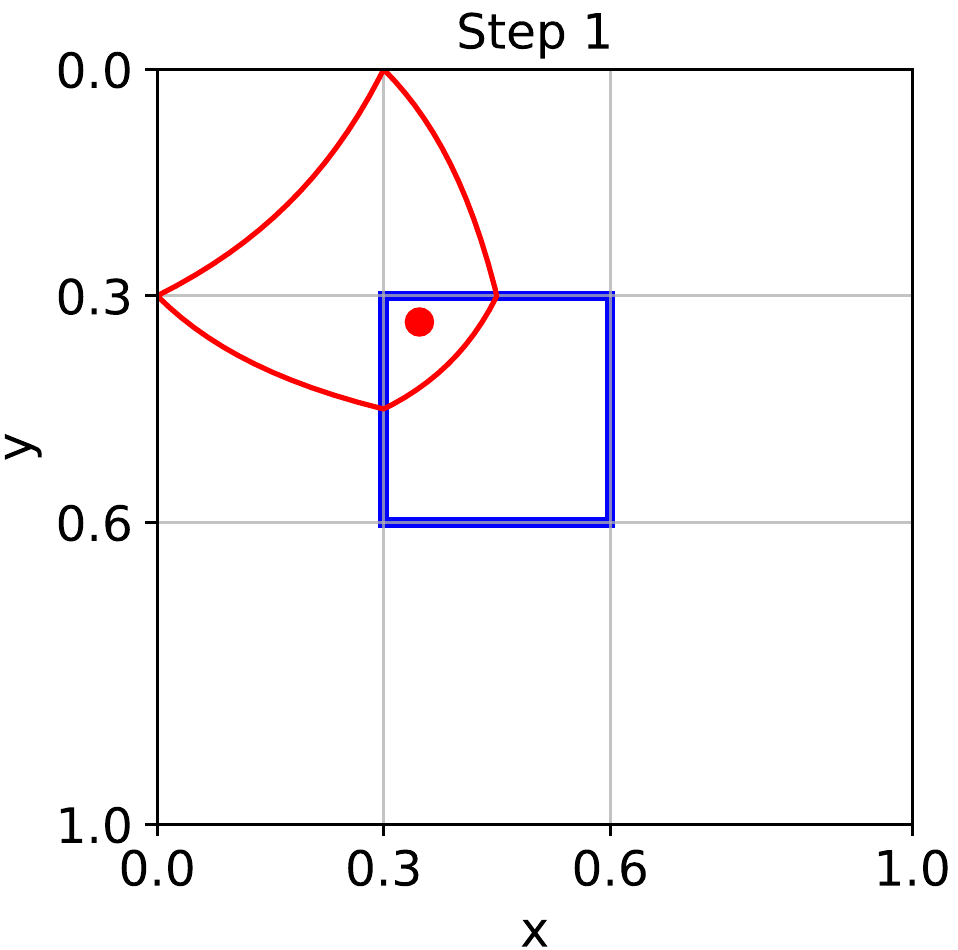} & 
\includegraphics[width=0.2\textwidth]{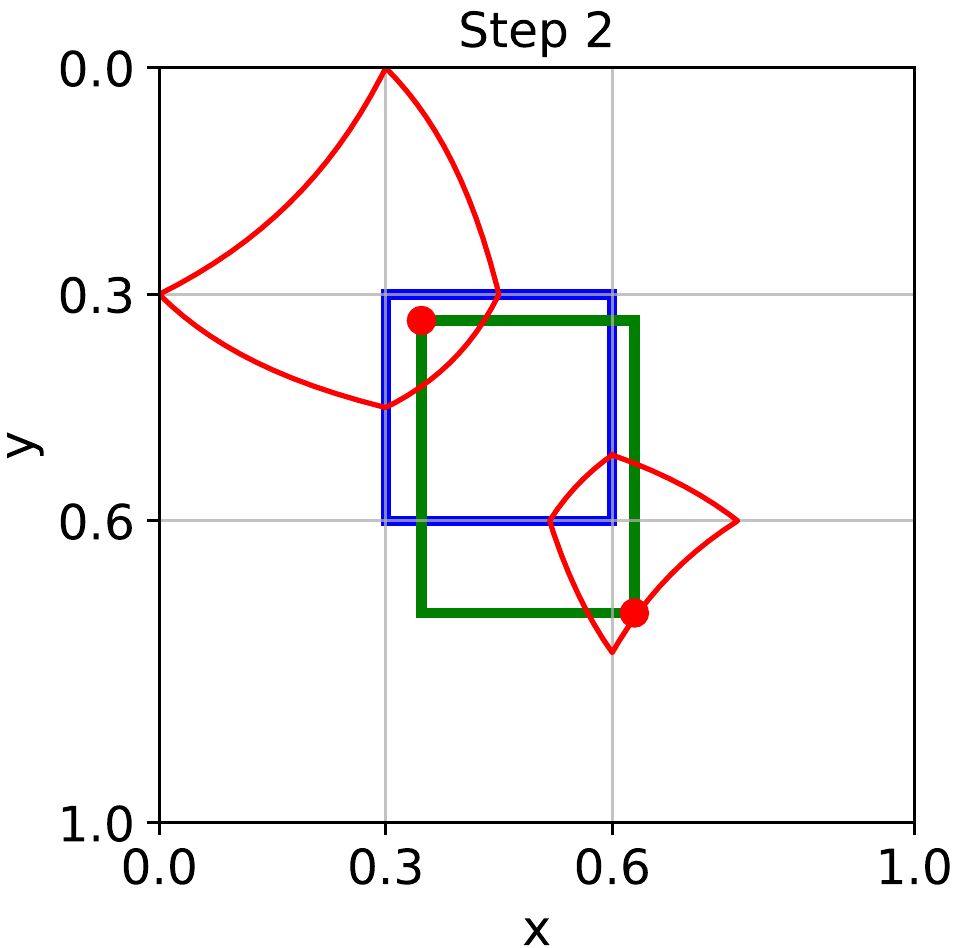} \\
\scriptsize (a) &\scriptsize (b) \\
\end{tabular}
\caption{(a,b) Applying Algorithm \ref{alg:IoUSample} on the blue BB ($B$) with $T=0.5$. Red polygons denote boundaries for top-left and bottom-right points that can be sampled with an IoU larger than $T=0.5$. Red dots are sampled points, and green box is the generated box ($\bar{B}$) with $IoU=0.5071$.}
\label{fig:Algorithm1}
\end{figure}

\begin{algorithm}
\caption{Bounding Box Generator. See Section \ref{sec:BBSampler} and the Appendix for the definitions of the functions. \label{alg:IoUSample}}
\begin{algorithmic}[1]
\footnotesize
\Procedure{GenerateBB}{$B, T$}
\State \textcolor{lightgray}{\# Step-1: Find top-left corner}
\State $TLPoly \gets \mathrm{findTLFeasibleSpace}(B, T)$ 
\State $\mathrm{TL}(\Bar{B}) \gets \mathrm{samplePolygon}(TLPoly)$ 
\State \textcolor{lightgray}{\# Step-2: Find bottom-right corner}
\State $ BRPoly \gets \mathrm{findBRFeasibleSpace}(B, T, \mathrm{TL}(\Bar{B}))$ 
\State $\mathrm{BR}(\Bar{B}) \gets \mathrm{samplePolygon}(BRPoly)$ 
\State \textbf{return} $[\mathrm{TL}(\Bar{B}),\mathrm{BR}(\Bar{B})]$ 
\EndProcedure
\end{algorithmic}
\end{algorithm}

Given a reference box $B$ and a threshold $T$, the goal of the bounding box (BB) generator is to determine a new box $\Bar{B}=[\Bar{x_1},\Bar{y_1},\Bar{x_2},\Bar{y_2}]$ such that $\mathrm{IoU}(B,\Bar{B}) \geq T$. To generate such a box, we propose a 2-step algorithm presented in Algorithm \ref{alg:IoUSample} and illustrated in Fig. \ref{fig:Algorithm1}. The first step (lines 3-4) finds the polygon\footnote{Note that the shape is not strictly a polygon; however, we approximate it as one at regular small intervals, and therefore, we call it a polygon for the sake of simplicity.} that computes the feasible space for $\mathrm{TL}(\Bar{B})=(\Bar{x_1},\Bar{y_1})$, which satisfies the desired IoU, and samples a point in this polygon. The second step (lines 6-7) takes into account the sampled $\mathrm{TL}(\Bar{B})$ and, similar to Step 1, determines a feasible space for bottom-right corner, then, samples   $\mathrm{BR}(\Bar{B})$. 

\begin{figure}{c}
\centering
\includegraphics[width=0.35\textwidth]{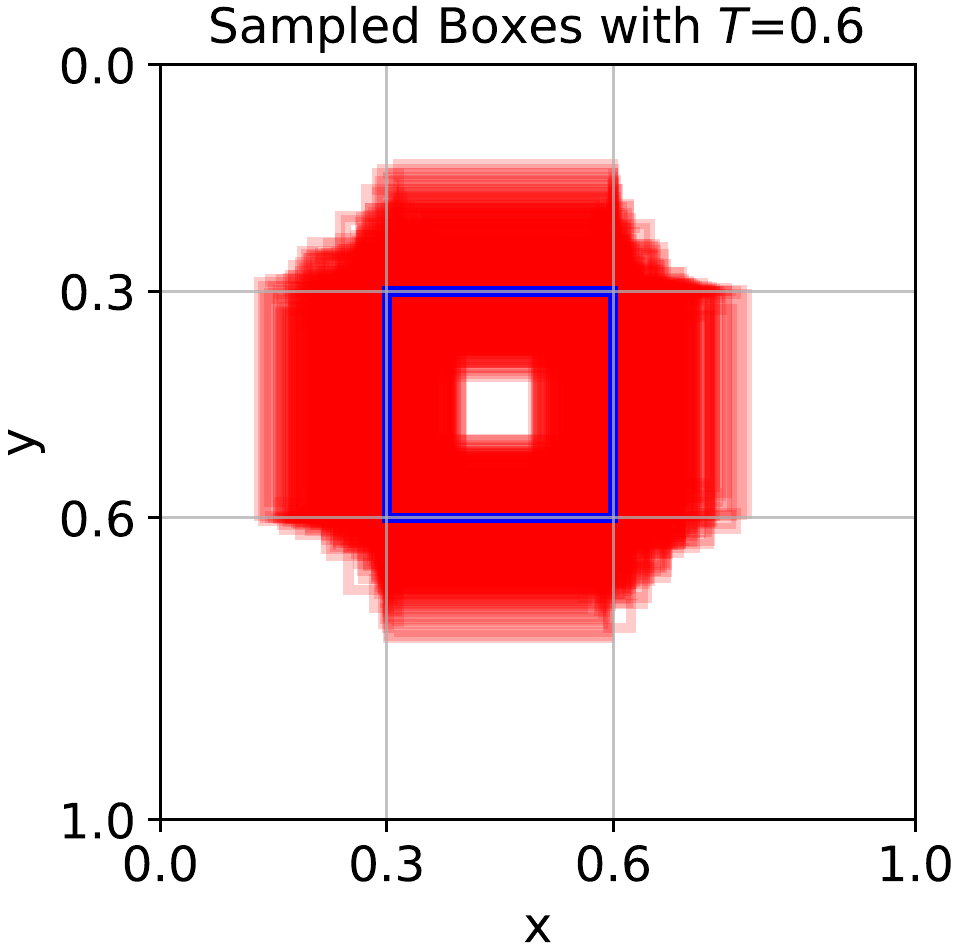} 
\caption{$1K$ generated boxes (shown with red) by Algorithm \ref{alg:IoUSample} for reference box drawn in blue ($B$) and IoU threshold $T=0.6$.}
\label{fig:ExBoxes}
\end{figure}
This order leads to a non-isotropic distribution with respect to the reference box. To make it isotropic, we can also sample in the reverse order: i.e. sample BR first then TL. We then randomly choose the order, before sampling. Fig. \ref{fig:ExBoxes} superimposes $1000$ generated boxes with $T=0.6$.

The following two sections discuss how the feasible space is computed (i.e. $\mathrm{findTLFeasibleSpace}(B, T)$) and how a point can be sampled within a polygon (i.e. $\mathrm{samplePolygon}(TLPoly)$). We refer the interested reader to check the Appendix for $\mathrm{BR}(\bar{B})$.

\subsubsection{Determining Feasible Space for the Desired IoU}
\label{subsub:TLSpace}


$\mathrm{findTLFeasibleSpace}(B, T)$ is the function determining the feasible set of points that can be the top left point of a box ensuring the desired IoU. In order to find the set of these feasible points (i.e. $\mathrm{TL}(\bar{B})$) that satisfy Eq. \ref{eq:IoU}, we assume that $\mathrm{BR}(\bar{B})=\mathrm{BR}(B)$ and manipulate Eq. \ref{eq:IoU}, otherwise, some feasible points are excluded in the feasible top left space. Even though $\mathrm{BR}(\bar{B})$ is fixed, there are still two unknown variables $\bar{x_1}$ and $\bar{y_1}$. That's why, we first bound one of these two variables and then find the value of the unbounded variable by moving within the limits of the bounded variable with some precision (we use $0.0001$ as precision). Since the definition of the $\mathrm{IoU}(B,\bar{B})$ is different in each of the four regions depicted in Fig. \ref{fig:Supp}(a) due to the $\max$ and $\min$ operations, an equation is to be derived for each region. 

\begin{figure}
\centering
\begin{tabular}{cc}
\includegraphics[width=0.2\textwidth]{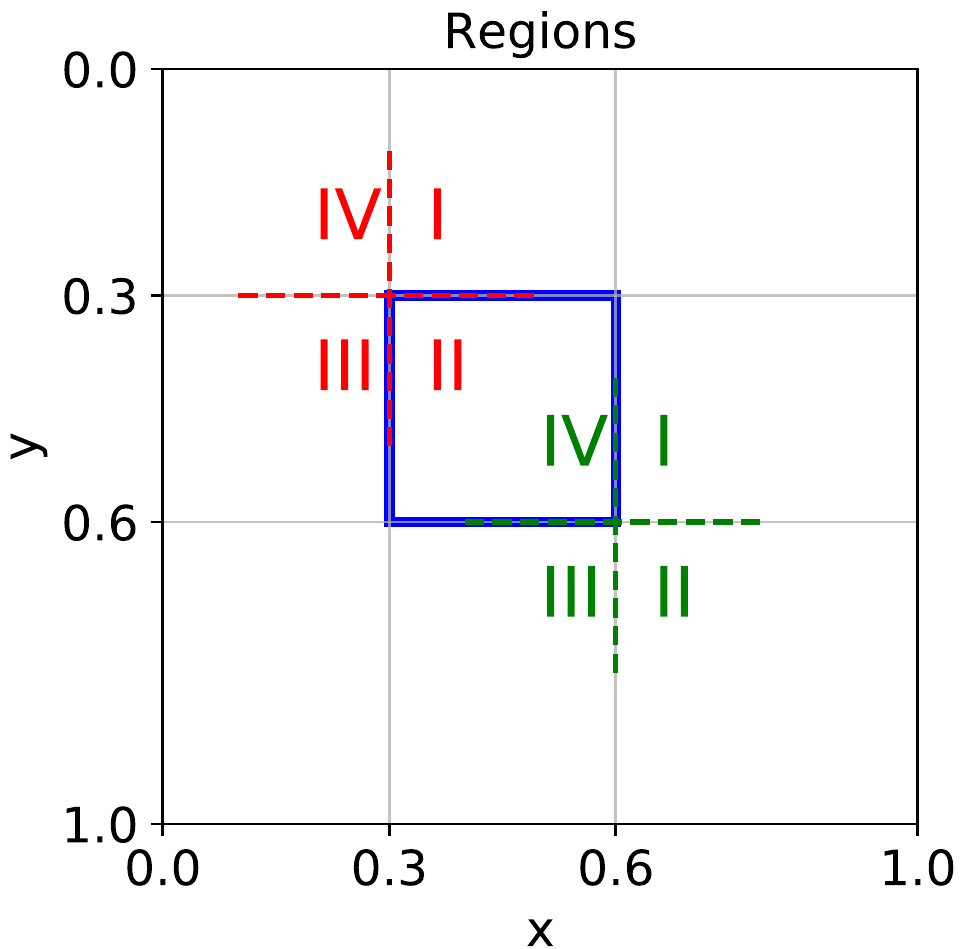} & 
\includegraphics[width=0.2\textwidth]{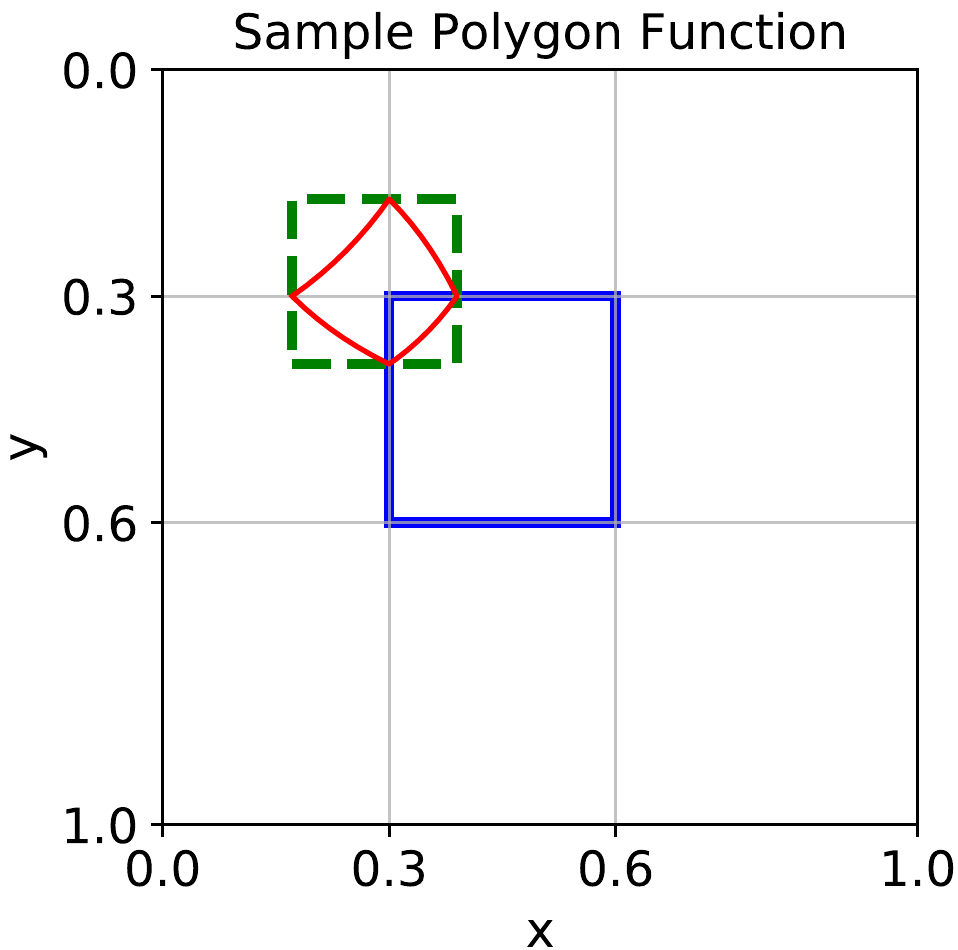} \\
\scriptsize (a) & \scriptsize (b)  \\
\end{tabular}
\caption{\textbf{(a)} The regions around $\mathrm{TL}(B)$ and $\mathrm{BR}(B)$ are splitted into four each. Red and green dashed lines split the top left and bottom right regions respectively. The numbers label the splitted regions.), \textbf{(b)} In the execution of the sample polygon function for $T=0.75$, green dashed box is the enclosing box for the TL space polygon.}
\label{fig:Supp}
\end{figure}

Denoting the minimum and maximum bounds of $\bar{x_1}$ in Region I by $x_{min}^I$ and $x_{max}^I$ respectively, we bound the values in $x$ axis. It is obvious that $x_{min}^I=x_1$ due to the boundary of Region I. To find $x_{max}^I$, we manipulate Eq. \ref{eq:IoU} by exploiting that $\Bar{y_1}=y_1$ for $x_{max}^I$, which yields:
\begin{align}
    x_{max}^I=x_2- (x_2-x_1)\times T.
\end{align}
Having determined the boundaries for $\Bar{x_1}$, now we derive a function that determines $\Bar{y_1}$ given $\Bar{x_1}$. Finally, moving within the bounds yields $\Bar{x_1},\Bar{y_1}$ pairs satisfying $\mathrm{IoU}(B,\bar{B})=T$ when $\mathrm{BR}(\bar{B})=\mathrm{BR}(B)$. In region I, note that $I(B,\bar{B})$ does not rely on $\Bar{y_1}$ (i.e. $I(B,\bar{B})=(x_2-\Bar{x_1})(y_2-y_1)$). Bringing these together, $\Bar{y_1}$ can be defined as (see Appendix for the entire derivation of $x_{max}^I$ and $\Bar{y_1}$):
\begin{align}
    \Bar{y_1}=y_2- \frac{ \frac{I(B,\Bar{B})}{T}+I(B,\Bar{B})-A(B)}{(x_2-\Bar{x_1})}.
\end{align}
Here, we only show the derivation steps for Region I and present the equations for all regions in Appendix. Combining the points in all these regions yields the polygon limiting feasible region with $IoU \geq T$.
\subsubsection{Controlling the Spatial Distribution of the Boxes}
\label{subsub:samplePolygon}
$\mathrm{samplePolygon}(TLPoly)$ function determines the BB spatial distribution. We follow rejection sampling \cite{TutorialMC} in such a way that a point is proposed by the proposal distribution until it hits the inside of the polygon. Accordingly, the proposal distribution determines the BB spatial distribution. Fig. \ref{fig:Supp}(b) presents an example for spatial uniform distribution for the top-left space polygon with $T=0.75$. We sample a point in the rectangle uniformly, which corresponds basically to generating two uniform numbers within a range. If the point is in the polygon, then it is accepted, else a new point is proposed until it is inside the polygon. Note that different proposal distributions lead to different BB spatial distributions.

\subsection{pRoI Generator: Training by Generated BBs}
\label{sec:PosRoISampler}

This section provides an application of our BB generator for generating positive RoIs for training a two-stage object detector. By applying our BB generator to the ground-truth boxes, we can generate positive RoIs with desired characteristics. This enables us to (i) analyze how the performance of Faster R-CNN is affected by the properties of the positive RoIs and (ii) improve the performance for IoU intervals where RPN is not able to generate enough samples. 


\begin{algorithm}[H]
\caption{Positive RoI Generator. See Section \ref{sec:PosRoISampler} and the Appendix for the definitions of functions $\mathrm{fgBalancedRoIAlloc}$ and $\mathrm{genRoIs}$.}
\label{alg:SampleBoxes}
\footnotesize
\begin{algorithmic}[1]
\Procedure{GeneratepRoI}{$GTs, \psi_{IoU}, W_{IoU},RoINum$}
\State $perGtRoI = \mathrm{fgBalancedRoIAlloc}(GTs, RoINum)$
\State $RoIs = \mathrm{genRoIs}(GTs,perGtRoI,\psi_{IoU},W_{IoU},RoINum)$
\State \textbf{return} $RoIs$
\EndProcedure
\end{algorithmic}
\end{algorithm}

The method, ``Positive RoI Generator'' (pRoI Generator), described in Algorithm \ref{alg:SampleBoxes}, can control several different characteristics of the set of positive RoIs. $\mathrm{fgBalancedRoIAlloc}()$ first divides $RoINum$ by the number of different classes in the given ground truth set, $GTs$, to determine the allocated box number per class, and then shares this value among each example of the same class equally. As a result, $\mathrm{fgBalancedRoIAlloc}()$ determines the number of boxes to be generated for each ground truth box in $GTs$. Secondly, given the allocated number of boxes for each ground truth, $\mathrm{genRoIs}()$ iteratively uses BB generator as a subroutine to provide a set of $RoINum$ RoIs. In this step the BB IoU distribution requirement is determined by the inputs $\psi_{IoU}$, the base of the IoU bins and the weight of the each bin denoted by $W_{IoU}$. $W_{IoU}$ is basically a multinomial distribution over the bins determined by $\psi_{IoU}$. An important benefit of pRoI generator is that training with the generated RoIs has no impact on the gradient flow for the training process (see Appendix). At each training iteration, RPN generates a set of RoIs among which we discard the positive ones and use the positive RoIs generated by the proposed method (Fig. \ref{fig:teaser}). Using our pRoI generator, we can address the imbalance problems regarding RoIs at three different levels:

\noindent\textbf{(1) Foreground-foreground class imbalance}, which occurs when a dataset or mini-batch (or batch) contains different numbers of positive examples from different classes. To illustrate on a batch, an image (used as a batch) from PASCAL dataset \cite{PASCAL} includes 4 bottles, 2 persons, 2 dining tables and 1 chair. In such a case, having equal number of RoIs per instance may lead the model to be biased in favor of the bottle class while ignoring the chair class. In our pRoI Generator, $\mathrm{fgBalancedRoIAlloc}()$ function allocates the same number of RoIs for each class within the batch.

\noindent\textbf{(2) BB IoU imbalance}, which occurs when the positive RoIs have a skewed IoU distribution (Fig. \ref{fig:IoUDistribution}). It has been shown that the hardness of a RoI is related to its IoU \cite{LibraRCNN} and also the regressor overfits to RoIs which has IoU around $0.5$ when the distribution of the RPN proposals is concentrated towards $0.5$ \cite{CascadeRCNN}. Thus, these recent findings imply that the IoU distribution has an important effect on training. As aforementioned, $\mathrm{genRoIs}()$ is able to control the IoU distribution of the BBs.

\noindent\textbf{(3) BB spatial imbalance}, which occurs when the BBs intersect significantly and a diverse set of examples can not be provided to the detection network. This level of imbalance is controlled in our pRoI generator in the subroutine BB generator as discussed in Section \ref{subsub:samplePolygon}.

\section{Experimental Setup}
{\noindent}\textbf{Dataset and Implementation Details:} We evaluate our generative methods on Faster R-CNN in two different settings: (i) on  Pascal VOC 2007 \cite{PASCAL} with backbone ResNet-101 following the implementation and training in \cite{jjfaster2rcnn} with batch size $1$ image on $1$ GPU, and (ii) on MS COCO \cite{COCO} with backbone ResNet-50 following the implementation and training in  \cite{MMdetection} with batch size $2$ images/GPU on $2$ GPUs. 


{\noindent}\textbf{Performance Measures:} We exhaustively search for the best mean-average-precision ($\mathrm{mAP}$) and mean-optimal-localization-precision-recall ($\mathrm{moLRP}$) error \cite{LRP} values over epochs and report them. moLRP is a recently introduced metric for object detection, which represents recall, precision and average tightness of the BBs. Note that $\mathrm{mAP}$ is a higher-is-better measure, while $\mathrm{moLRP}$ is an error metric and thus, it is a lower-is-better measure.

{\noindent}\textbf{RoI Sources:} In addition to RoIs output by RPN, we use the RoIs generated by our pRoI generator, with a given distribution,  during the analysis and training. The different distributions are obtained by controlling $W_{IoU}$ (see Appendix for the exact configurations of $W_{IoU}$) in Algorithm \ref{alg:SampleBoxes}. Unless otherwise stated, we set $\psi_{IoU}=[0.5,0.6,0.7,0.8,0.9]$ and $RoINum=32$. We train these RoI sources with and without foreground balanced sampling in order to see the effects of different imbalance problems on different RoI sources. The results are presented in Table \ref{table:BatchProperties}.

\section{Imbalance Problems and Analysis of RPN RoIs}
\label{sec:Analysis}

\begin{figure}{c}
\centering
\includegraphics[width=0.35\textwidth]{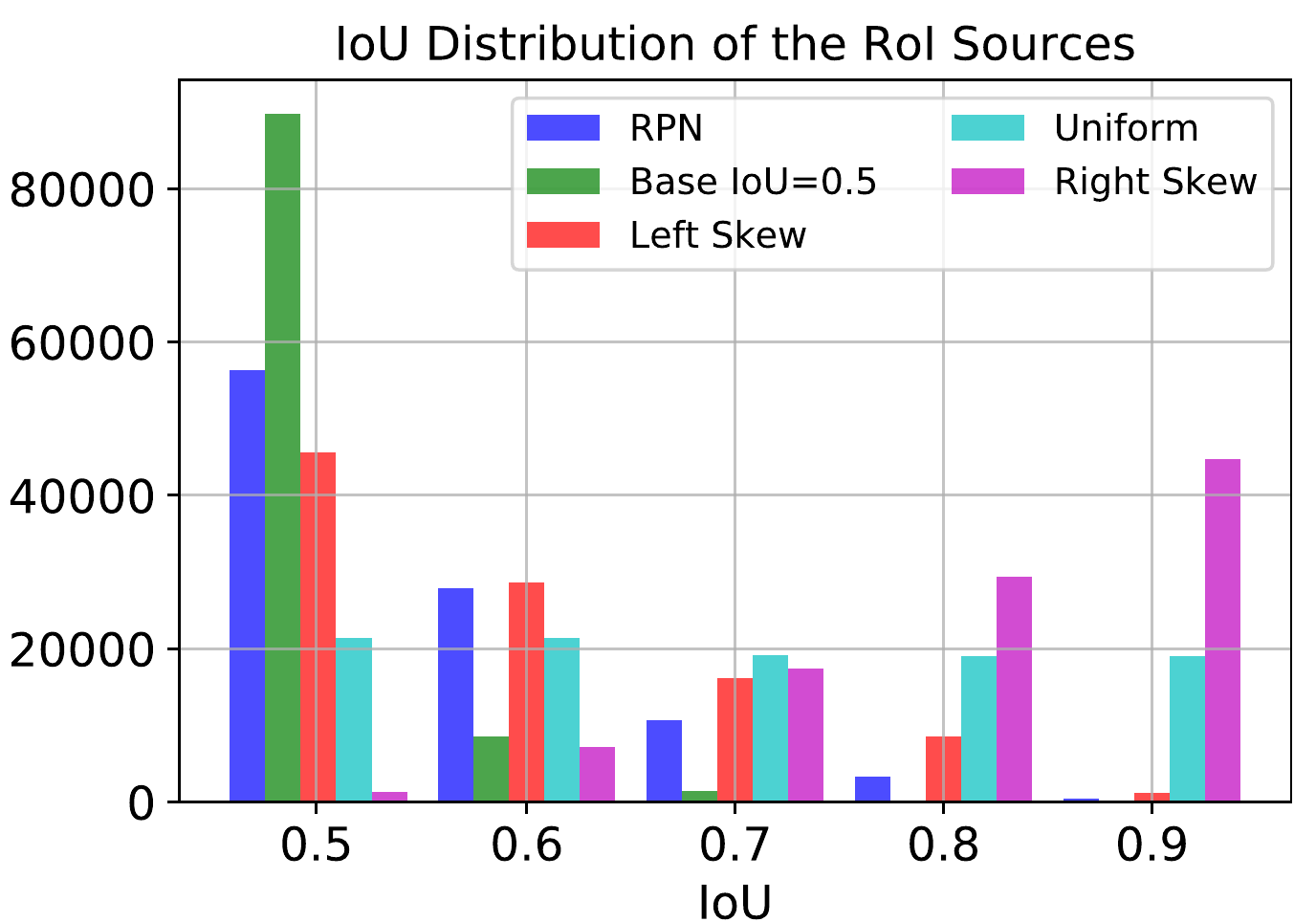} 
\caption{IoU distribution of different RoI Sources. See Appendix for the configurations of the RoI sources.
}
\label{fig:IoUDistribution}
\end{figure}

In this section, we first point out some imbalance on the distribution of BBs. Then, we investigate how several characteristics of RoIs affect detection performance by generating RoIs with our pRoI generator.


\begin{table}
\caption{Effect of the batch properties for generated positive samples (see Fig. \ref{fig:IoUDistribution} for different RoI sources) on Pascal VOC 2007. We trained each RoI source with balanced foreground-foreground distribution and simulating OHPM. RS, Unif, LS and Base respectively denote pRoI-Right Skew, pRoI-Uniform, pRoI-Left Skew and pRoI-Base IoU=0.5 distributions. FGB refers to foreground balanced generation of RoIs.}
 \centering
 \scriptsize
 \setlength{\tabcolsep}{0.2em}
\begin{tabular}{|l|c|c|c|c|c|c|c|}
\hline
RoI &$FGB?$ &OHPM& $\mathrm{moLRP}$ & $\mathrm{moLRP}$ & $\mathrm{moLRP}$ & $\mathrm{moLRP}$ & $\mathrm{mAP@{0.5}}$\\
Distrib. &  & & $\downarrow$ & (IoU) $\downarrow$ & (FP) $\downarrow$ & (FN) $\downarrow$ & $\uparrow$\\
\hline\hline
&No&No&$64.6$&$21.4$&$18.7$&$29.8$&$74.9$\\
RS &Yes&No&$64.5$&$21.5$&$18.7$&$29.5$&$75.3$\\
&Yes&Yes&$60.4$&$19.5$&$16.8$&\textbf{27.2}&$77.4$\\\hline
&No&No&$61.3$&$19.5$&$17.9$& $28.5$&$76.3$\\
Unif. &Yes&No&$61.1$&$19.5$&$17.0$&$28.8$&$76.9$\\
&Yes&Yes& \textbf{59.9}&$19.2$&\textbf{16.0}&$27.6$&\textbf{77.8}\\ \hline  
&No&No&$60.4$&$19.1$&$16.9$& $28.3$&$77.0 $\\
LS &Yes&No&$60.3$&\textbf{19.0}&$17.3$&$28.2$&$77.2$\\
&Yes&Yes&$60.7$&$19.3$&$17.7$& $27.8$&$76.9$\\\hline
&No&No&$61.5$&$19.7$&$17.2$& $28.8$&$76.6$\\ 
Base &Yes&No&$61.4$&$19.3$&$16.3$&$29.4$&$76.7$\\
&Yes&Yes&$61.2$&$19.7$&$16.6$&$28.6$&$76.7$\\
\hline
\end{tabular}
\label{table:BatchProperties}
\end{table}

\subsection{BB IoU Imbalance}

Our BB generator method (Algorithm \ref{alg:IoUSample}) samples boxes for a given IoU threshold, spatially uniformly. It does not impose an upper bound for the IoUs of the sampled boxes. Therefore, in order to analyze the density of the different IoUs for the positive samples, we uniformly generate $100K$ boxes for each IoU distribution type and plot the distribution of the generated boxes in Fig. \ref{fig:IoUDistribution}. Note that training a detector with different IoU distributions of positive examples affects the resulting test performance (see Table \ref{table:BatchProperties}), which implies the effect of BB IoU imbalance.

From Fig. \ref{fig:IoUDistribution}, we observe the following:

\textbf{(1)} The distribution of the boxes with $base IoU=0.5$ is highly biased towards $0.5$ and includes very low samples with higher IoUs. This implies that the proportion of the boxes with $IoU>0.9$ is far too low than that of the boxes with $0.6>IoU>0.5$ when $T=0.5$. 

\textbf{(2)} RPN RoIs follow a similar tendency to the sampled boxes with $base IoU=0.5$ since the RoIs are based on anchors, which are uniformly distributed with a fixed set of boxes on the image. Thanks to the RPN regressor, the IoU distribution improves compared to the distribution of the sampled boxes with $base IoU=0.5$. On the other hand, this bias towards $0.5$ is previously argued to make the regressor overfit for smaller IoUs \cite{CascadeRCNN}.

\textbf{(3)} RPN is able to provide hard positive examples inherently; however, the number of prime samples (i.e. examples with larger IoUs) is quite low. This is critical since it is shown that prime sampling performs better than hard positive mining \cite{PrimeSample}.

\subsection{Foreground-Foreground Class Imbalance}
We observe that, for each RoI source, addressing foreground-foreground imbalance ($fg\_balance=1$) improves performance in terms of both mAP and moLRP, especially for the right skew and uniform cases (See Table \ref{table:BatchProperties}). Moreover, addressing foreground-foreground class imbalance does not seem to affect the localization error ($\mathrm{moLRP_{IoU}}$) but improves the classification performance since $\mathrm{mAP@{0.5}}$, $\mathrm{moLRP_{FP}}$ and $\mathrm{moLRP_{FN}}$ get better (except for the left-skew case). Therefore, we conclude that foreground-foreground class imbalance can also be alleviated by employing methods in the batch level.

\subsection{Effect of Online Hard Positive Mining}
Here we demonstrate another useful use-case of our pRoI generator by simulating OHEM \cite{OHEM} on positive examples. OHEM chooses the positive and negative examples with the highest loss values after applying NMS to the examples to preserve example diversity. A recent study \cite{LibraRCNN} showed that the IoU and the hardness of an example are correlated. On the other hand, another study \cite{PrimeSample} proposed an opposite perspective to the OHEM based on prioritizing ``prime samples'', i.e. samples with high IoUs. To be more clear, OHEM \cite{OHEM} implies preferring positive examples with IoUs just above $0.5$, while prime sampling asserts that the higher the IoU, the better the example. To make an analysis on the positive examples, we simulate OHEM by (i) initially generating $128$ BBs by pRoI generator, (ii) applying NMS using loss value of an example, (iii) finally selecting the ones with the larger loss values. We coin this as \textbf{online hard positive mining (OHPM)}. OHPM also presents an example where pRoI generator can simulate sampling schemes.

In our experiments, we also observe that the effect of the hard examples depends on the IoU distribution of the RoIs and high-quality samples are required during training: In Table \ref{table:BatchProperties}, when OHPM is applied, uniform and right-skew distributions, which have more difficult examples due to their distribution (Fig. \ref{fig:IoUDistribution}), have better performance compared to the left-skew and ``Base IoU=0.5'' cases. Moreover, while OHPM does not improve the performance of left-skew and ``Base IoU=0.5'' cases, it is crucial for the right-skew and uniform distributions (Table \ref{table:BatchProperties}). Therefore, similar to prime sampling \cite{PrimeSample}, we show that examples with higher IoUs are crucial during training, however, we also show that these examples should be supported by hard examples. 

\subsection{BB Spatial Imbalance}

\begin{figure}{c}
\centering
\includegraphics[width=0.35\textwidth]{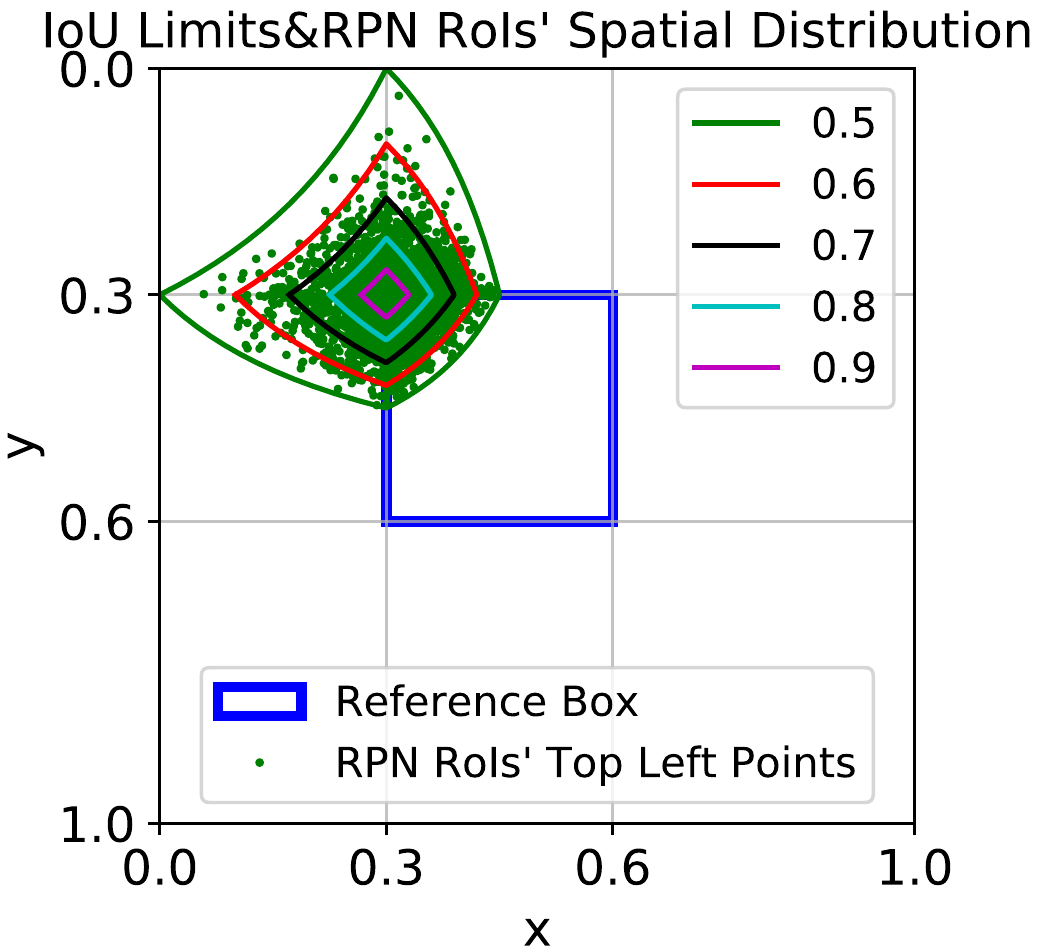}
\caption{Spatial Distribution of the top left points of $2,500$ RPN RoIs and maximum IoU Limits from $IoU=0.9$ to $0.5$ (in-out direction)}
\label{fig:IoUSptDistribution}
\end{figure}
We now analyze the spatial distribution of the RPN RoIs and how they fit within the theoretical IoU boundaries in Fig. \ref{fig:IoUSptDistribution}. To be able to make such an analysis, we selected a reference box with $[x_1, y_1, x_2, y_2] = [0.3,0.3,0.6,0.6]$. At the final epoch of the RPN training, we track positive RPN proposals and their associated ground truths. As discussed in Section \ref{sec:BBSampler}, we scaled and shifted the ground truths to the reference box and applied the same transformations to  their associated positive RPN proposals (i.e., RoIs). Among the positive RPN proposals, top-left (TL) points of the $2,500$ RoIs are plotted with green dots in Fig. \ref{fig:IoUSptDistribution}. Then, using $\mathrm{findTLFeasibleSpace}()$ function in Algorithm \ref{alg:IoUSample}, we plot the theoretical limits for the top left points for RoIs with IoUs larger than $0.5$, $0.6$, $0.7$, $0.8$ and $0.9$.

Fig. \ref{fig:IoUSptDistribution} leads to several key findings: \textbf{(1)} As expected, as the IoU decreases, the boundaries occupy a larger space around the TL point of the reference box. A result of this is that the sample space for $0.9$ is very small, which makes it more difficult to have distinct RoIs with $IoU>0.9$. 
\textbf{(2)} We observe that no TL point is outside of the $0.5$ boundary, which is a sanity check for the boundaries since a RoI is labeled as positive if it has at least $0.5$ IoU with a ground truth. 
\textbf{(3)} The TL points of the RPN RoIs are accumulated around the TL point of the reference box and they are not uniformly distributed within the $0.5$ boundary. 
\textbf{(4)} The TL points of the RPN RoIs tend to be inside the reference box more than to be outside. Specifically, RPN RoIs between $x>0.3, y>0.3$ and $x<0.3, y<0.3$ are $28.2\%$ and $21.0\%$ of the all, respectively. 

Especially the last two observations may be critical for an object detector since they may result in a positive bias towards specific RoIs and may make the generalization difficult over the entire spatial space. However, the effects of all these observations require experimental or theoretical validation that is not provided in this paper.

\section{Practical Improvements}
\label{section:Practical}

\begin{table}
\caption{Average performance of 3 runs for Faster R-CNN with our OFB sampling on Pascal VOC. Lower is better for $\mathrm{moLRP}$ and its components, whereas higher is better for mAP.}
 \centering
 \scriptsize 
\begin{tabular}{|l|c|c|c|c|c|c|}
\hline
Sampling& &  \multicolumn{3}{c|}{$\mathrm{moLRP}$} & \\
Method& $\mathrm{moLRP} \downarrow$ & $\mathrm{{IoU}}\downarrow$ & $\mathrm{{FP}}\downarrow$ & $\mathrm{{FN}}\downarrow$ & $\mathrm{mAP@{0.5}}\uparrow$\\
\hline\hline
Random & $59.4$&$18.7$&$16.2$&$27.7$&$78.0$\\
OFB & \textbf{58.9}&$18.7$& \textbf{15.6}& \textbf{27.2}& \textbf{78.5}\\
  \hline
\end{tabular}
\label{table:classImb}
 \end{table}

\begin{table}
\caption{Comparison of different sampling mechanisms on MS COCO using Faster R-CNN. Lower is better for $\mathrm{moLRP}$ and its components, whereas higher is better for mAP. mAP stands for COCO-style mAP. R and H denote random and hard sampling respectively, and OFB is our sampling method for positive RoIs. The first block compares among different positive sampling schemes combined with random sampling, while the second block compares their combinations with hard example mining.}
 \centering
 \scriptsize 
\begin{tabular}{|c|c|c|c|c|}
\hline
\multicolumn{2}{|c|}{Sampling Method}& & & \\
Positive&Negative& $\mathrm{moLRP} \downarrow$ & $\mathrm{mAP}\uparrow$& $\mathrm{mAP@{0.5}}\uparrow$\\
\hline\hline
R&R&$72.4$&$34.1$&$55.2$\\
H&R&$75.3$&$31.0$&$51.7$\\
OFB&R &\textbf{72.1}&\textbf{34.7}&\textbf{55.8}\\ \hline
R&H&$71.9$&$35.3$&$54.6$\\
H&H&$74.6$&$31.1$&$50.0$\\
OFB&H&\textbf{70.9}&\textbf{35.6}&\textbf{55.3}\\
  \hline
\end{tabular}
\label{table:classImbCOCO}
 \end{table}
In this section, we use OFB sampling and BB generator to improve an object detector by addressing foreground-class imbalance and by controlling the number and distribution of RoIs for training the second-stage.

\subsection{Online Foreground Balanced Sampling} 
In end-to-end training, the set of positive RoIs are limited and they are not generated as in pRoI generator. Motivated from the analysis using pRoI generator on the effect of foreground-foreground class imbalance (see Section \ref{sec:Analysis}), we propose an online sampling method to be used in the conventional training pipeline. Denoting the total number of classes in a batch by $C$ and the number of positive RoIs for class $c$ by $k_c$, each RoI is assigned a  probability $1/(C k_c)$ and the subset of RoIs to train Faster R-CNN is sampled from this multinomial distribution. We call this sampling scheme as Online Foreground Balanced (OFB) Sampling.

In order to see the effect, we train Faster R-CNN with and without OFB sampling and present results in Tables \ref{table:classImb} and \ref{table:classImbCOCO}. For the Pascal VOC \cite{PASCAL}, we observe $0.5\%$ improvement in $\mathrm{mAP@{0.5}}$ and $\mathrm{moLRP}$, with better performance in precision and recall components of $\mathrm{moLRP}$ and no impact on the regression branch. In our experiments with MS COCO (Table \ref{table:classImbCOCO}), we compared our results with hard example mining \cite{SSD, OHEM}. Similar to the findings of Cao et al. \cite{PrimeSample} and our analysis in Section \ref{sec:Analysis}, while hard positive mining does not improve performance, our OFB sampling is beneficial for foreground examples. Moreover, the table shows that OFB sampler can be combined with sampling approaches for negative BBs. In any case, similar to our experiments for Pascal VOC, the best performance gain is in $\mathrm{mAP@{0.5}}$. This suggests that controlling RoIs to balance foreground classes has also a role during training of the object detectors and OFB, an efficient sampling algorithm, can be considered a basic solution for the problem.



\subsection{Generating More Samples in Higher IoUs} 

\begin{table*}
\caption{Performance Comparison with RPN on PASCAL VOC. $RoINum$ is the input of pRoI generator, fg/bg is the desired fg and bg RoI numbers during training, and Mean RoI \# is the actual mean of number of positive RoIs. Note that fg/bg RoI numbers are set differently for pRoI and RPN so that the best performance is achieved for both of these RoI sources in order to provide a fair comparison especially in favor of RPN. We trained the models (except the one with the $^*$ mark) for 16 epochs with a learning rate decay at epochs $9$ and $14$ since our model provides more diverse data than RPN (see in Fig. \ref{fig:IoUSptDistribution} that the TL points of the RPN RoIs clusters around TL point of $B$) and there are fewer samples for training in higher IoUs (see Mean RoI \# in Table \ref{table:Performance})}
 \centering
 \scriptsize 
\begin{tabular}{|l|c|c|c|c|c|c|c|c|c|}
\hline
RoI Source&$IoU$&$RoINum$&$fg/bg$&Mean RoI \# $\uparrow$&$\mathrm{moLRP} \downarrow$&$\mathrm{moLRP_{IoU}} \downarrow$& $\mathrm{moLRP_{FP}} \downarrow$& $\mathrm{moLRP_{FN}} \downarrow$&$\mathrm{mAP@IoU} \uparrow$\\
\hline\hline
RPN*&$0.5$&N/A&$32/96$&\textbf{27.12}&$59.3$&$18.7$&$16.0$&\textbf{27.7}&\textbf{78.0}\\
pRoI-Uniform&$0.5$&$128$&$32/96$&$25.49$&\textbf{59.2}&\textbf{18.4}&\textbf{15.5}&$28.2$ &$77.1$\\ \hline
RPN&$0.6$&N/A&$27/81$&$16.92$&$65.4$&$17.0$&\textbf{19.4}&$31.9$&\textbf{71.2}\\
pRoI-Uniform&$0.6$&$128$&$27/81$&\textbf{18.28}&$65.4$&\textbf{16.9}&$20.8$&\textbf{31.0}&$70.6$\\
\hline
RPN&$0.7$&N/A&$9/27$&$5.39$&$74.9$&\textbf{14.7}&\textbf{27.2}&$42.1$&$57.3$\\
pRoI-Uniform&$0.7$&$128$&$18/54$&\textbf{9.93}&\textbf{74.5}&$14.9$&$28.0$&\textbf{39.8}&\textbf{57.5}\\
\hline
RPN&$0.8$&N/A&$2/6$&$1.08$&$92.5$&$13.2$&$58.8$&$69.8$&$21.3$\\
pRoI-Uniform&$0.8$&$64$&$8/24$&\textbf{3.92}&\textbf{87.7}&\textbf{12.1}&\textbf{47.8}&\textbf{59.3}&\textbf{32.2}\\
  \hline
RPN&$0.9$&N/A&$2/6$&$0.17$&$99.5$&$7.4$&$94.2$&$97.1$&$0.5$\\
pRoI-Uniform&$0.9$&$32$&$2/6$&\textbf{1.62}&\textbf{99.3}&\textbf{7.3}&\textbf{92.4}&\textbf{96.0}&\textbf{0.9}\\
  \hline
\end{tabular}
\label{table:Performance}
 \end{table*}
 
 \begin{table*}
\caption{Effect of $RoINum$ on PASCAL VOC. Speeds are reported on a single Geforce GTX 1080 Ti.}
 \centering
 \scriptsize 
\begin{tabular}{|l|c|c|c|c|c|c|c|c|c|}
\hline
RoI Source&$RoINum$& $\mathrm{moLRP} \downarrow$ & $\mathrm{moLRP_{IoU}} \downarrow$& $\mathrm{moLRP_{FP}} \downarrow$& $\mathrm{moLRP_{FN}} \downarrow$ & $\mathrm{mAP@{0.5}} \uparrow$&Train Speed $\downarrow$&Mean RoI \# $\uparrow$\\
\hline\hline
pRoI-Uniform&$32$&$60.3$&$19.3$&$16.4$&$27.8$&$77.5$&$0.41s$&$14.81$\\
pRoI-Uniform&$64$&$59.7$&$19.0$&$16.1$&$27.4$&$77.6$&$0.58s$&$21.32$\\
pRoI-Uniform&$128$&$59.9$&$19.2$&$16.0$&$27.6$&$77.8$&$0.97s$&$25.49$\\
  \hline
\end{tabular}
\label{table:RoINumEffect}
 \end{table*}
 
Our approach can be integrated into an object detector without any hindrance on the gradient paths (see Appendix). In this section, we compare a detector trained with our pRoI Generator with a detector trained with the conventional method (i.e. using RPN RoIs) -- see Table \ref{table:Performance}. We use Uniform RoI source with foreground balance and OHPM since it performed the best in Table \ref{table:BatchProperties}. 
For $IoU=\Theta$, we randomly sample negative samples from the output of the RPN in the range $[0.1, \Theta]$ and the positive samples are provided by the pRoI generator also using OHPM. To apply OHPM, we first generate $RoINum$ boxes, then select $fg$ many from them. In IoUs $0.6-0.8$, for which fewer RoIs are possible than $0.5$, we initially train the models for $1$ epoch by setting $fg=32$ and $bg=96$ and track ``Mean RoI \#'' to see an upper bound for the models to generate RoIs and prevent class imbalance modelwise. In this run, Mean RoI \# for IoUs $0.6, 0.7, 0.8$ are $17.26, 7.60, 1.72$ for RPN and $20.0, 11.41, 4.67$ for pRoI-Uniform respectively. Then using $IoU=0.5$ as an example, we multiply the resulting ``Mean RoI \#'' by $1.5$ and set $fg$ approximately to it with $bg=3\times fg$ as in the conventional training. This approach makes training more stable and fair especially for the RPN (see Table \ref{table:Performance}) by balancing foreground and background consistently. 
 
Looking at Table \ref{table:Performance} and comparing the methods in the IoUs that they are trained for, we observe the following:
\textbf{(1)} For $IoU = 0.5$, $0.6$ and $IoU=0.7$ we get comparable results with the conventional training. 
\textbf{(2)} For $IoU=0.8$, where RPN is not able to generate sufficient samples, the performance increases significantly in terms of both metrics since, at each iteration, generated positive boxes are provided consistently to the second stage.
\textbf{(3)} Overall, the mean RoI \# is approximately four times higher at $IoU=0.8$; and, $\mathrm{mAP@0.8}$ and $\mathrm{moLRP}$ improve by $10.9\%$ and $4.8\%$ respectively. A similar trend is also achieved for $IoU=0.9$.
 
In short, these results demonstrate that it is possible to train an object detector using BB generator with comparable results for lower IoUs and significantly better performance for higher IoUs. On par performance for low IoUs can be owing to the fact that there are sufficient amount of samples for these cases to see any imbalance effect.

{\noindent}\textbf{Effect of $\bm{RoINum}$:} Apart from the input parameters to determine the nature of the RoI source, $RoINum$ is the only new hyperparameter in Algorithm \ref{alg:SampleBoxes}. In Table \ref{table:RoINumEffect}, we observe that training improves ($\mathrm{mAP}$ increases) when $RoINum$ is increased because we have more positive samples at each iteration. However, more samples mean slower (yet still acceptable) training speed compared to conventional training having $0.23s$ training speed. 

{\noindent}\textbf{Preliminary Results on MS COCO:} In order to back up our claims, we also conducted an experiment on MS COCO dataset using $IoU=0.8$ with Faster R-CNN. Compared to the baseline achieving  $\mathrm{moLRP}=95.1$ and $\mathrm{mAP@0.8}=13.2$, using pRoI generator the model has $\mathrm{moLRP}=93.7$ and $\mathrm{mAP@0.8}=15.3$. These results suggest that our model is able to generate more diverse examples than the baseline in larger IoUs.
\section{Conclusion}

In this paper, we proposed a BB generator and a positive RoI generator. We showed that generated RoIs can be used both as an analysis tool (owing to its controllable nature) and a training method for the two-stage object detectors. 

We showed that there is a bias in the RPN RoIs' IoU and spatial distribution with respect to the IoU boundaries that are physically possible and analyzed the IoU distributions of RPN and other RoI sources. 

Using our BB generator, we developed a pRoI generator that can generate RoIs overlapping with a GT box with a desired IoU or spatial distribution. Then, we trained Faster R-CNN's second-stage with the RoIs generated according to different distributions. We showed that, by producing more samples than RPN, we can achieve better or comparable performance to Faster R-CNN. Moreover, our results reconciliated two conflicting recent studies \cite{PrimeSample,OHEM} that both using high-IoU RoIs and hard examples can have positive effect on the training if the IoU distribution is appropriate.

Our ideas can be used for analyzing the anchors of a one-stage detector (as well as those of a two-stage detector) in order to design a better anchor set. Furthermore, other applications, e.g. tracking, that require spatially distributed BBs with certain properties can also exploit our approach.
\section*{Acknowledgments}
This work was partially supported by the Scientific and Technological Research Council of Turkey (T\"UB\.ITAK) through the project titled ``Object Detection in Videos with Deep Neural Networks'' (grant number 117E054). Kemal \"Oks\"uz is supported by the T\"UB\.ITAK 2211-A National Scholarship Programme for Ph.D. students. The experiments were partially performed at T\"UB\.ITAK ULAKBIM, High Performance and Grid Computing Center (TRUBA) resources.






\section*{Appendix}
\setcounter{section}{0}
\section{The Properties of $\textrm{IoU}(B,\Bar{B})$}
\label{sec:intro}
Following upon the notation in Section 3 of the paper, we introduce the following properties. For clarity we assume that intersection of two boxes is greater than $0$ and the last pixel is not taken into account (i.e. instead of $\mathrm{A}(B)=(x_2-x_1+1)$, we adopted $\mathrm{A}(B)=(x_2-x_1)$).

\begin{theorem}
$\mathrm{IoU}(B,\Bar{B})$ is scale-invariant.
\end{theorem}
\begin{proof}
Assume that $k_x>0$ and $k_y>0$ are the scaling factors in the $x$ and $y$ axes respectively and $B_s$, $\Bar{B}_s$ are the scaled boxes. We show that $\mathrm{IoU}(B_s,\Bar{B}_s) =\mathrm{IoU}(B,\Bar{B})$ as follows:
\begin{strip}
\begin{align}
\label{eq:IoUDefSc}
\mathrm{IoU}(B_s,\Bar{B}_s) &= \frac{I(B_s,\Bar{B}_s)}{A(B_s)+A(\Bar{B}_s)-I(B_s,\Bar{B}_s)} \\ \label{eq:DefArIntSc}
&=\frac{\left( \min{(k_x \Bar{x_2},k_x x_2)}-\max{(k_x \Bar{x_1},k_x x_1)} \right)\times \left( \min{(k_y \Bar{y_2},k_y y_2)}-\max{(k_y \Bar{y_1},k_y y_1)} \right)}{(k_x x_2-k_x x_1)\times(k_y y_2-k_y y_1)+(k_x \Bar{x_2}-k_x \Bar{x_1})\times(k_y \Bar{y_2}-k_y \Bar{y_1})-I(B_s,\Bar{B}_s)} \\
\label{eq:MulSc}
&=\frac{k_x \left( \min{(\Bar{x_2},x_2)}-\max{(\Bar{x_1},x_1)} \right) \times k_y \left( \min{(\Bar{y_2},y_2)}-\max{(\Bar{y_1},y_1)} \right)}{k_x (x_2-x_1)\times k_y (y_2-y_1)+k_x (\Bar{x_2}-\Bar{x_1})\times k_y (\Bar{y_2}-\Bar{y_1})-I(kB,\Bar{kB})} \\
\label{eq:DefIntSc}
&=\frac{k_x k_y I(B,\Bar{B}) }{k_x k_y (x_2-x_1)\times(y_2-y_1)+k_x k_y (\Bar{x_2}-\Bar{x_1})\times(\Bar{y_2}-\Bar{y_1})-k_x k_y I(B,\Bar{B})} \\
\label{eq:DefFinalSc}
&=\frac{k_x k_y I(B,\Bar{B}) }{k_x k_y \left( (x_2-x_1)\times(y_2-y_1)+(\Bar{x_2}-\Bar{x_1})\times(\Bar{y_2}-\Bar{y_1})-I(B,\Bar{B}) \right) } \\
&=\mathrm{IoU}(B,\Bar{B})
\end{align}
\end{strip}
Eq. \ref{eq:IoUDefSc} defines the IoU and Eq. \ref{eq:DefArIntSc} replaces area and intersection definitions. In Eq. \ref{eq:MulSc}, we use the property that multiplying by a positive scalar does not change minimum and maximum of two numbers. Eq. \ref{eq:DefIntSc} incorporates the intersection definition. Eq. \ref{eq:DefFinalSc} gets the denominator in the $k \hat{k}$ parenthesis, which simplifies the term to the definition of $IoU(B,\Bar{B})$.
\end{proof}

\begin{theorem}
$\textrm{IoU}(B,\Bar{B})$ is translation-invariant.
\end{theorem}
\begin{proof}
Assuming that $k_x \in  \mathrm{R}$ and $k_y \in  \mathrm{R}$ are the perturbation in the $x$ and $y$ axis respectively and $B_t$, $\Bar{B}_t$ are the perturbed boxes. We show that $\mathrm{IoU}(B_t,\Bar{B}_t) =\mathrm{IoU}(B,\Bar{B})$ as follows:
\newpage
\footnotesize
\begin{strip}
\begin{align} 
\label{eq:IoUDefSh}
\mathrm{IoU}(B_t,\Bar{B}_t) &= \frac{I(B_t,\Bar{B}_t)}{A(B_t)+A(\Bar{B_t})-I(B_t,\Bar{B}_t)} \\ 
\label{eq:DefArIntSh}
&=\frac{\left( \min{(\Bar{x_2}+k_x,x_2+k_x)}-\max{(\Bar{x_1}+k_x,x_1+k_x)} \right)\times\left( \min{(\Bar{y_2}+k_y,y_2+k_y)}-\max{(\Bar{y_1}+k_y,y_1+k_y)} \right)}{((x_2+k_x)-(x_1+k_x))\times((y_2+k_y)-(y_1+k_y))+((\Bar{x_2}+k_x)-(\Bar{x_1}+k_x))*((\Bar{y_2}+k_y)-(\Bar{y_1}+k_y))-I(B_t,\Bar{B}_t)}\\
\label{eq:AddSc}
&=\frac{\left( \min{(\Bar{x_2},x_2)}+k_x-\max{(\Bar{x_1},x_1)-k_x} \right)\times\left( \min{(\Bar{y_2},y_2)}+k_y-\max{(\Bar{y_1},y_1)}-k_y \right)}{(x_2+k_x-x_1-k_x)\times(y_2+k_y-y_1-k_y)+(\Bar{x_2}+k_x-\Bar{x_1}-k_x)\times(\Bar{y_2}+k_y-\Bar{y_1}-k_y)-I(B_t,\Bar{B}_t)}\\
\label{eq:CanstCancel}
&=\frac{\left( \min{(\Bar{x_2},x_2)}-\max{(\Bar{x_1},x_1)} \right)\times\left( \min{(\Bar{y_2},y_2)}-\max{(\Bar{y_1},y_1)} \right)}{(x_2-x_1)\times(y_2-y_1)+(\Bar{x_2}-\Bar{x_1})\times(\Bar{y_2}-\Bar{y_1})-I(B_t,\Bar{B}_t)}\\
\label{eq:DefFinalSh}
&=\frac{I(B,\Bar{B})}{A(B)+A(\Bar{B})-I(B,\Bar{B})}\\
&=\mathrm{IoU}(B,\Bar{B})
\end{align}
\end{strip}
Again, Eq. \ref{eq:DefArIntSh} replaces area and intersection definitions in the IoU definition. In Eq. \ref{eq:AddSc}, we use the property that adding a scalar to numbers adds a scalar to the minimum and maximum of two numbers. In Eq. \ref{eq:CanstCancel}, constants cancel each other and Eq. \ref{eq:DefFinalSh} replaces area and intersection for the $\mathrm{IoU}(B,\Bar{B})$, which simplifies to the definition of $\mathrm{IoU}(B,\Bar{B})$.
\end{proof}

\section{Details of the Bounding Box Generator}
\label{sec:IoUSampler}
In this section we present the derivation of the Equation 4 and 5, and explain the $\mathrm{findBRFeasibleSpace}(B, T, \mathrm{TL}(\Bar{B}))$ function. 

\begin{table*}
\caption{Top-Left space bounds and equations. See Fig. 4 in the paper.}
 \centering
 \scriptsize 
\begin{adjustbox}{width=\textwidth}
\begin{tabular}{|c|c|c|c|}
\hline
Region&Min Bound&Max Bound&Equation\\
\hline\hline
I&$\Bar{x_1}=x_1$&$\Bar{x_1}=x_2- (x_2-x_1)\times T$&$\Bar{y_1}=y_2- \frac{ \frac{I(B,\Bar{B})}{T}+I(B,\Bar{B})-A(B)}{(x_2-\Bar{x_1})} $\\
\hline
II&$\Bar{y_1}=y_1$&$\Bar{y_1}=y_2-\frac{A(B)\times T}{x_2-x_1}$
&$\Bar{x_1}=x_2-\frac{I(B,\Bar{B})\times A(B)}{(y_2-\Bar{y_1})}$\\
\hline
III&$\Bar{y_1}=y_1$&$\Bar{y_1}=y_2-\frac{A(B)\times T}{x_2-x_1}$
&$\Bar{x_1}=x_2-\frac{ \frac{I(B,\Bar{B})}{T}-A(B)+I(B,\Bar{B})}{(y_2-\Bar{y_1})}$\\
\hline
IV&$\Bar{y_1}=\frac{(y_2\times(T-1))+ y_1}{T}$&$\Bar{y_1}=y_1$&    $\Bar{x_1}=x_2-\frac{A(B)}{T\times(y_2-\Bar{y_1})}$\\
  \hline
\end{tabular}
\end{adjustbox}
\label{table:TLSpaceEq}
 \end{table*}

\begin{table*}
\caption{Bottom-Right space bounds.}
 \centering
\resizebox{\columnwidth*2}{!}{
\begin{tabular}{|c|c|c|}
\hline 
Region&Min Bound&Max Bound\\
\hline \hline
I&$\Bar{y_2}=\frac{T\times A(B)+T\times(x_2-\alpha)\times\beta+\beta\times(x_2-\alpha)-T\times\Bar{y_1}\times(x_2-\Bar{x_1})}{((T+1)\times(x_2-\alpha)-T\times(x_2-\Bar{x_1}))} $&$\Bar{y_2}=y_2$\\
\hline
II&$\Bar{x_2}=x_2$&$\Bar{x_2}=\Bar{x_1}+\frac{ \frac{I(B,\Bar{B})}{T}-A(B)+I(B,\Bar{B})}{(y_2-\Bar{y_1})}$\\
\hline
III&$\Bar{y_2}=y_2$&$\Bar{y_2}=\Bar{y_1}+\frac{ \frac{I(B,\Bar{B})}{T}-A(B)+I(B,\Bar{B})}{(x_2-\Bar{x_1})}$
\\
\hline
IV&$\Bar{x_2}=\frac{T\times A(B)+T\times(y_2-\beta)\times\alpha+\alpha\times(y_2-\beta)-T\times\Bar{x_1}\times(y_2-\Bar{y_1})}{((T+1)\times(y_2-\beta)-T\times(y_2-\Bar{y_1}))} $&$\Bar{x_2}=x_2$\\
  \hline
\end{tabular}}
\label{table:BRSpaceLim}
 \end{table*}
  \begin{table*}
\caption{Bottom-right space equations.}
 \centering
\resizebox{400pt}{!}{
\begin{tabular}{|c|c|}
\hline
Region & Equation\\
\hline\hline
I&$\Bar{x_2}=\Bar{x_1}+\frac{\frac{I(B,\Bar{B})}{T}-A(B)+I(B,\Bar{B})}{\Bar{y_2}-\Bar{y_1}}$\\
\hline
II&$\Bar{y_2}=\Bar{y_1}+\frac{\frac{I(B,\Bar{B})}{T}-A(B)+I(B,\Bar{B})}{\Bar{x_2}-\Bar{x_1}}$\\
\hline
III&$\Bar{x_2}=\frac{T\times A(B)+\alpha\times T\times (\hat{\beta}-\beta)+\alpha\times(\hat{\beta}-\beta)-T\times\Bar{x_1}\times(\Bar{y_2}-\Bar{y_1})}{(T+1)\times(\hat{\beta}-\beta)-T\times(\Bar{y_2}-\Bar{y_1}))}$\\
\hline
IV&$\Bar{y_2}=\frac{T\times A(B)+\beta\times T\times(\hat{\alpha}-\alpha)+\beta\times(\hat{\alpha}-\alpha)-T\times\Bar{y_1}\times(\Bar{x_2}-\Bar{x_1})}{(T+1)\times(\hat{\alpha}-\alpha)-T\times(\Bar{x_2}-\Bar{x_1}))}$\\
  \hline
\end{tabular}}
\label{table:BRSpaceEq}
 \end{table*}
\begin{figure*}[h]
\centering
\includegraphics[width=0.85\textwidth]{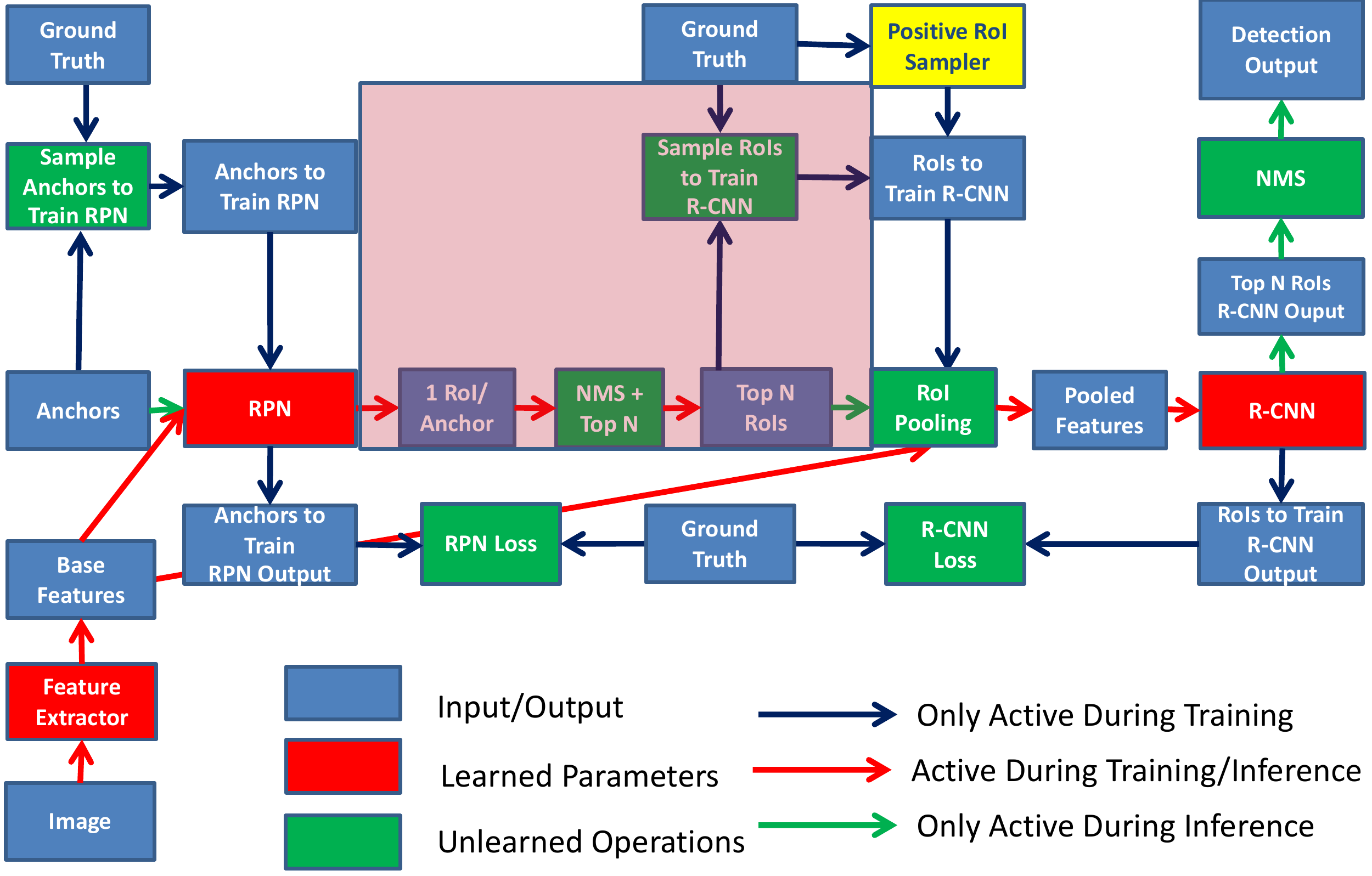} 
\caption{Conventional Faster R-CNN Training and our modification. During training, Positive RPN RoIs are not utilized and thus the modules presented under the large red rectangle are not used for positive RoIs. These RoIs are generated by the Positive RoI Generator shown in yellow box.}
\label{fig:overview}
\end{figure*}

\subsection{$\mathrm{findTLFeasibleSpace}(B, T)$ function}
Here, we derive Equation 4 and 5 in the paper, and present the equations for the top-left space..

In order to derive Equation 4 depicting $x_{max}^I$, we bound the $x$ coordinate first. It is obvious that $x_{min}^I=x_1$ due to the boundary of Region I. For $x_{max}^I$, we know that $\Bar{y_1}=y_1$ again thanks to the region boundary. Therefore, since we have only one unknown, $x_{max}^I$, we use Eq. the definition of the IoU to determine its value in Eq. \ref{eq:defIoU4}-\ref{eq:final1}. Eq. \ref{eq:defIoU1} defines IoU based on Eq. \ref{eq:defIoU4}. In Eq. \ref{eq:putTerms1}, we set $\min{(\Bar{x_2},x_2)}=x_2$, $\max{(\Bar{x_1},x_1)}=x_{max}^I$, $\min{(\Bar{y_2},y_2)}=y_2$ and $\max{(\Bar{y_1},y_1)}=y_1$ by taking into the intersection definition in Region I. Also note that $\Bar{x_1}=x_{max}^I$, $\Bar{y_1}=y_1$, $\Bar{x_2}=x_2$ and $\Bar{y_2}=y_2$ in this case. In Eq. \ref{eq:rearrTerms1}-\ref{eq:final1}, we just rearrange the terms to have $x_{max}^I$ as a left hand side term.
\begin{strip}
\begin{align}
    \label{eq:defIoU4}
    \textrm{IoU}(B,\Bar{B})&= \frac{I(B,\Bar{B})}{A(B)+A(\Bar{B})-I(B,\Bar{B})} \\ 
    \label{eq:defIoU1}
    &=\frac{\left( \min{(\Bar{x_2},x_2)}-\max{(\Bar{x_1},x_1)} \right)\times\left( \min{(\Bar{y_2},y_2)}-\max{(\Bar{y_1},y_1)} \right)}{(x_2-x_1)\times(y_2-y_1)+(\Bar{x_2}-\Bar{x_1})\times(\Bar{y_2}-\Bar{y_1})-I(B,\Bar{B})} \\
    \label{eq:putTerms1}
    &\Rightarrow T=\frac{( x_2-x_{max}^I)\times( y_2-y_1)}{(x_2-x_1)\times(y_2-y_1)+(x_2-x_{max}^I)\times(y_2-y_1)-( x_2-x_{max}^I)\times( y_2-y_1)} \\    
    \label{eq:rearrTerms1}    
    &\Rightarrow  (x_2-x_1)\times (y_2-y_1)\times T=(x_2-x_{max}^I)\times ( y_2-y_1) \\ 
    \label{eq:rerearrTerms1}    
    &\Rightarrow x_{max}^I=x_2-\frac{(x_2-x_1)\times (y_2-y_1)\times T}{( y_2-y_1)}\\
    \label{eq:final1}    
    &\Rightarrow  x_{max}^I=x_2- (x_2-x_1)\times T
\end{align}
\end{strip}
Now since we know the values of $\Bar{x_1}$ based on the bounds, we can derive the Equation 5 (in the paper) for any $\Bar{y_1}$ value in equations by moving within bounds. Since $I(B,\bar{B})=(x_2-\Bar{x_1})\times (y_2-y_1)$, it does not rely on $\Bar{y_1}$ and we directly use $I(B,\bar{B})$ in the following equations:

\begin{align}
    &\mathrm{IoU}(B,\Bar{B})=\frac{I(B,\Bar{B})}{A(B)+(x_2-\Bar{x_1})\times (y_2-\Bar{y_1})-I(B,\Bar{B})} \\
    &\Rightarrow T\times (x_2-\Bar{x_1})\times (y_2-\Bar{y_1})\nonumber \\
    &=I(B,\Bar{B})+T\times I(B,\Bar{B})-T\times A(B)\\
    &\Rightarrow \Bar{y_1}=y_2- \frac{ \frac{I(B,\Bar{B})}{T}+I(B,\Bar{B})-A(B)}{(x_2-\Bar{x_1})}
\end{align}
Table \ref{table:TLSpaceEq} presents all of the equations derived using the same methodology.

\subsection{$\mathrm{findBRFeasibleSpace}(B, T, \mathrm{TL}(\Bar{B}))$ Function}
We follow the same approach for the bottom right corner with the top left corner. However, different from top-left space this step is required also consider the point generated top-left point. Note that the size of the polygon in the bottom-right space is affected by the distance between $\mathrm{TL}(\Bar{B})$ and $\mathrm{TL}(B)$. Maximum bottom-right polygon size, with exactly the same size of the top-left polygon, is achieved when $\mathrm{TL}(\Bar{B})=\mathrm{TL}(B)$. Conversely, bottom-right polygon degenerates to a point at $\mathrm{BR}(B)$ if the sampled $\mathrm{TL}(\Bar{B})$ hits the border of the top-left polygon. 

We add two additional parameters for the sake of clarity: $\alpha = \max(\bar{x_1}, x_1)$, $\beta= \max(\bar{y_1}, y_1)$, $\hat{\alpha} = \min(\bar{x_2}, x_2)$ and $\hat{\beta}=\min(\bar{y_2}, y_2)$. The bounds and the equations are derived by the same methodology that is illustrated in the first step presented in Tables \ref{table:BRSpaceLim} and \ref{table:BRSpaceEq} respectively. 
 
\section{Implementation Details}
\label{sec:ImplementationDetails}
\subsection{Integrating pRoI Generator into the Training}
The training of the two-stage object detectors involves 3 different networks as shown in Fig. \ref{fig:overview}. The first network is the feature extractor (i.e. ResNet\cite{Resnet}) which presents the base features to the second network, the proposal generator (i.e. RPN \cite{FasterRCNN}), and the third network, which is the object detector (i.e. R-CNN \cite{RCNN}, R-FCN \cite{RFCN}). The feature extractor is trained with the gradients back-propagated from the proposal generator and the object detector. The proposal generator is trained by a subset of the anchor-ground truth combinations (chosen by Sample Anchors to Train RPN in Figure \ref{fig:overview}) and a subset of these RPN proposals (i.e. RoIs) (chosen by Sample RoIs to Train R-CNN in Figure \ref{fig:overview}) are fed into the R-CNN after a series of operations including NMS and RoI Pooling that do not include learnable parameters. Finally, the loss is back-propagated through the entire network to update the parameters. However, the RoIs from the RPN is limited in number and diversity, which can impact the analysis and training. To address this, pRoI generator aims to generate RoIs with any desired property and in any number. Note that the gradients can also be back-propagated to the feature extractor as in the conventional training (i.e. RPN) since positive RoI Generator uses ground truths to generate an RoI in a similar manner to the conventional training, but differently it can generate boxes with the desired properties. During training, only for positive RoIs, pRoI Generator does not use the modules that are under the transparent red rectangle in Figure \ref{fig:overview}. However, during test time, our method follows the conventional approach, namely the RoIs from RPN are used due to the fact that no ground truth information is available during testing.

\subsection{Connection Between $\mathrm{genRoIs}()$ and $\mathrm{generateBB}()$}
As described in the text, $\mathrm{generateBB}()$ is a low-level function and any approach uses generated BBs approach is to rely on this function. That`s why it is a subroutine of $\mathrm{genRoIs}()$. The main idea in our implementation is to generate bounding boxes by iteratively calling the $\mathrm{generateBB}()$ for $RoINum$ times.

Apart from $RoINum$, the number of RoIs to be generated, there are two main input sets to the $\mathrm{genRoIs}$ function. Firstly, $GTs$ and $perGTRoI$ together have the information about the box coordinates and the number of RoIs to be generated from each ground truth box. Therefore, for $i^{th}$ ground truth box ($GTs_i$), we call $\mathrm{generateBB}()$ function for $perGTRoI_i$ times. And the second set of input comprises $\psi_{IoU}$ and $W_{IoU}$, which together have information about the weights of each IoU interval. Therefore, for determining an IoU for each box, we first generate $perGTRoI$ number of samples of IoU intervals using the multinomial distribution defined in $W_{IoU}$ and then, for each resulting interval, we again sample uniformly an IoU within its limits. These IoUs are clipped from $0.95$ in order to prevent the problems arising from the precision problem in the $\mathrm{samplePolygon}()$ acceptance process. This sampling strategy distributes the input IoUs over an interval evenly. Finally, we randomly shuffle this set of IoUs and associate them to the ground truths, which completes the generation of the ground truth and desired IoU pairs as the input of the $\mathrm{generateBB}()$ function.

\begin{table}
\caption{The configurations of $W_{IoU}$ for the different tables in the paper.}
 \centering
\resizebox{\columnwidth}{!}{
\begin{tabular}{|c|c|c|c|c|c|c|c|}
\hline
Table&RoI Source& $IoU=0.5$& $IoU=0.6$ & $IoU=0.7$ & $IoU=0.8$ & $IoU=0.9$\\
\hline
1&Right Skew&$0.02$&$0.10$&$0.20$&$0.30$&$0.38$\\
1&Balanced&$0.33$&$0.17$&$0.18$&$0.17$&$0.15$\\
1&Left Skew&$0.73$&$0.12$&$0.15$&$0.05$&$0$\\
4&Balanced, IoU=0.5&$0.33$&$0.17$&$0.18$&$0.17$&$0.15$\\
4&Balanced, IoU=0.6&$0$&$0.38$ &$0.20$&$0.22$&$0.20$\\
4&Balanced, IoU=0.7&$0$&$0$&$0.48$&$0.25$&$0.27$\\
4&Balanced, IoU=0.8&$0$&$0$&$0$&$0.64$&$0.36$\\
4&Balanced, IoU=0.9&$0$&$0$&$0$&$0$&$1$\\
  \hline
\end{tabular}
}
\label{table:Configurations}
 \end{table}
\subsection{Configurations of $W_{IoU}$}

The configurations of the $W_{IoU}$ (i.e. the distribution over $\psi_{IoU}=[0.5,0.6,0.7,0.8,0.9]$) used for the experiments is shown in Table \ref{table:Configurations}.


{\small
\bibliographystyle{ieee}
\bibliography{proposalbibliography}
}


\end{document}